\documentclass[american,a4paper, 10pt, conference]{article}
\usepackage[T1]{fontenc}
\usepackage[latin9]{inputenc}
\usepackage{amsthm}
\usepackage{amsmath}
\usepackage{amssymb}
\usepackage{graphicx}

\makeatletter

\providecommand{\tabularnewline}{\\}

 \usepackage[ruled,vlined]{algorithm2e}
\theoremstyle{plain}
\newtheorem{thm}{\protect\theoremname}
  \theoremstyle{definition}
  \newtheorem{defn}[thm]{\protect\definitionname}
  \theoremstyle{plain}
  \newtheorem{prop}[thm]{\protect\propositionname}
  \theoremstyle{plain}
  \newtheorem{lem}[thm]{\protect\lemmaname}
  \theoremstyle{plain}
  \newtheorem*{thm*}{\protect\theoremname}

\usepackage{aaai}

\usepackage{times}
\usepackage{helvet}
\usepackage{courier}
\usepackage{xcolor}
\usepackage{bm}

\@ifundefined{showcaptionsetup}{}{%
 \PassOptionsToPackage{caption=false}{subfig}}
\usepackage{subfig}
\makeatother

  \providecommand{\definitionname}{Definition}
  \providecommand{\lemmaname}{Lemma}
  \providecommand{\propositionname}{Proposition}
  \providecommand{\theoremname}{Theorem}
\providecommand{\theoremname}{Theorem}

\begin{document}

\title{Complete Decentralized Method for On-Line Multi-Robot Trajectory
Planning in Valid Infrastructures}

\author{Michal \v{C}áp \and Ji\v{r}í Vok\v{r}ínek\\
ATG, Dept. of Computer Science,\\
FEL, CTU in Prague, Czech Republic\\
\And Alexander Kleiner\\
iRobot Inc.,\\
Pasadena, CA, USA}
\maketitle
\begin{abstract}
We consider a system consisting of multiple mobile robots in which
the user can at any time issue relocation tasks ordering one of the
robots to move from its current location to a given destination location.
In this paper, we deal with the problem of finding a trajectory for
each such relocation task that avoids collisions with other robots.
The chosen robot plans its trajectory so as to avoid collision with
other robots executing tasks that were issued earlier. We prove that
if all possible destinations of the relocation tasks satisfy so-called
valid infrastructure property, then this mechanism is guaranteed to
always succeed and provide a trajectory for the robot that reaches
the destination without colliding with any other robot. The time-complexity
of the approach on a fixed space-time discretization is only quadratic
in the number of robots. We demonstrate the applicability of the presented
method on several real-world maps and compare its performance against
a popular reactive approach that attempts to solve the collisions
locally. Besides being dead-lock free, the presented approach generates
trajectories that are significantly faster (up to 48\% improvement)
than the trajectories resulting from local collision avoidance. 
\end{abstract}

\section{Introduction}

Consider a future factory where intermediate products are moved between
workplaces by autonomous robots. The worker at a particular workplace
calls a robot, puts an object to a basket mounted on the robot and
orders the robot to autonomously deliver the object to another workspace
where the object will be retrieved by a different worker. Clearly,
an important requirement on such a system is that each robot must
be able to avoid collisions with other robots autonomously operating
in the shared space. The problem of avoiding collisions between individual
robots can be approached either from a control engineering perspective
by employing reactive collision avoidance or from AI perspective by
planning coordinated trajectories for the robots.

In the reactive approach, the robot follows the shortest path to its
current destination and attempts to resolve collision situations as
they appear. Each robot periodically observes positions and velocities
of other robots in its neighborhood. If there is a potential future
collision, the robot attempts to avert the collision by adjusting
its immediate heading and velocity. A number of methods have been
proposed~\cite{vanDenBerg2008RVO,Guy2009_ClearPath,lalish2012distributed}
that prescribe how to compute such collision-avoiding velocity in
a reciprocal multi-robot setting, with the most prominent one being
ORCA~\cite{vanBerg2011_ORCA}. These approaches are widely used in
practice thanks to their computational efficiency -- a collision-avoiding
velocity for a robot can be computed in a fraction of a millisecond~\cite{vanBerg2011_ORCA}.
However, these approaches resolve collisions only locally and thus
they cannot guarantee that the resulting motion will be deadlock-free
and that the robot will always reach its destination.

With a planning approach, the system first searches for a set of globally
coordinated collision-free trajectories from the start position to
the destination of each robot. After the planning has finished, the
robots start following their respective trajectories. If robots are
executing the resulting joint plan precisely (or within some predefined
tolerance), it is guaranteed that the robots will reach their destination
while avoiding collisions with other robots. It is however known that
the problem of finding coordinated trajectories for a number of mobile
objects from the given start configurations to the given goal configurations
is intractable. More precisely, the coordination of disks amidst polygonal
obstacles is NP-hard~\cite{SpirakisY84_Strong_NP_Hardness_of_Moving_Many_Discs}
and the coordination of rectangles in a bounded room is PSPACE-hard~\cite{hopcroft84}. 

Even though the problem is relatively straightforward to formulate
as a planning problem in the Cartesian product of the state spaces
of the individual robots, the solutions can be very difficult to find
using standard heuristic search techniques as the joint state-space
grows exponentially with increasing number of robots. The complexity
can be partly mitigated using techniques such as ID~\cite{Standley10}
or M{*}~\cite{Wagner2014AIJ_subdimensionalExpansion} that solve
independent sub-conflicts separately, but each such sub-conflict can
still be prohibitively large to solve because the time complexity
of the planning is still exponential in the number of robots involved
in the sub-conflict.

Instead, heuristic approaches are often used in practice, such as
prioritized planning~\cite{Erdmann87onmultiple}, where the robots
are ordered into a sequence and plan one-by-one such that each robot
avoids collisions with the higher-priority robots. This greedy approach
tends to perform well in uncluttered environments, but it is in general
incomplete and often fails in complex environments. 

When the geometric constraints are ignored, complete polynomial algorithms
can be designed such as Push\&Rotate~\cite{deWilde_push_and_rotate_aamas}
or Bibox~\cite{Surynek:2009:NAP:1703435.1703586}. These algorithms
solve so-called ``Pebble motion'' problem, in which pebbles(robots)
move on a given graph such that each pebble occupies exactly one vertex
and no two pebbles can occupy the same vertex or travel on the same
edge during one timestep. Although this model can be useful for coordination
of identical robots on coarse graphs%
\footnote{The graph must be coarse enough so that the bodies of two robots ``sitting''
on two different vertices will never overlap. The same has to hold
for two robots traveling on different edges of the graph.%
}, it is not applicable for trajectory coordination of robots with
fine-grained or otherwise rich motion models. 

Recall that in our factory scenario, a robot can be assigned a task
``online'' at any time during the operation of the system. If one
of the classical planners is used, the system would have to interrupt
all the robots and replan their trajectories each time a new task
is assigned, which is clearly undesirable. Further, although prioritized
planning can be run in a decentralized manner~\cite{VelagapudiSS10,cap_2013_b},
the complete approaches are difficult to run without a central solver.

The main contribution of this paper is COBRA -- a novel decentralized
method for collision avoidance in multi-robot systems with online-assigned
tasks to individual robots. We prove that for robots operating in
environments that were designed as a \emph{valid infrastructure}s,
the method is complete, i.e. all tasks are guaranteed to be successfully
carried out without collision. Furthermore, if a time-extended roadmap
planner is used for trajectory planning, the algorithm has polynomial\textbf{
}worst-case\textbf{ }time-complexity in the number of robots. The
applicability of the approach is demonstrated using a simulated multi-robot
system operating in real-world environments.

\section{Problem Statement}

Consider a 2-d environment (described by a set of obstacle-free coordinates
$\mathcal{W}\subseteq\mathbb{R}^{2}$) and a set of points $E\subset\mathcal{W}$
representing distinguished locations in the environment called endpoints
(modeling e.g. workplaces in a factory). The pair $(\mathcal{W},E)$
is called and infrastructure. Such an infrastructure is populated
by $n$ mobile robots with circular bodies indexed $1,\ldots,n$.
The radius of the body of robot $i$ is denoted $r_{i}$, the maximum
speed robot $i$ can move at is denoted by $v_{i}^{\mathrm{max}}$.
The robots are assumed to be holo\-nomic, i.e. at every time step,
the robot can select its immediate speed $v\in(0,v_{i}^{\mathrm{max}})$
and heading $\theta\in(-\pi,\pi)$ with the acceleration limits being
neglected. During the operation of the system, robots are assigned
\emph{relocation tasks} denoted $s\rightarrow g$ requesting the chosen
robot to move from its current position $s\in E$ to the given goal
endpoint $g\in E$. We assume that the robot cannot be interrupted
once it starts executing a particular relocation task and thus a new
relocation task can be assigned to a robot only after it has reached
the destination of the previously assigned task. The objective is
to find a trajectory for each such relocation task such that the robot
will reach the specified goal without colliding with other robots
operating in the system. Moreover, such trajectories should be found
in a decentralized fashion without a need for a central component
coordinating individual robots.

\subsubsection*{Notation}

In the remainder of the paper we will make use of the notion of a\emph{
space-time region}: When a spatial object, such as the body of a
robot, follows a given trajectory, then it can be thought of as occupying
a certain region in space-time $\mathcal{T}:=\mathcal{W}\times\left[0,\infty\right)$.
A dynamic obstacle $\Delta$ is then a region in such a space-time
$\mathcal{T}$. If $(x,y,t)\in\Delta$, then we know that the spatial
position $(x,y)$ is occupied by dynamic obstacle $\Delta$ at time
$t$. The function
\[
R_{i}^{\Delta}(\pi):=\left\{ (x,y,t):t\in[0,\infty)\wedge(x,y)\in R_{i}(\pi(t))\right\} 
\]
 maps trajectories of a robot $i$ to regions of space-time that the
robot $i$ occupies when its center point follows given trajectory
$\pi$. As a special case, let $R_{i}^{\Delta}(\emptyset):=\emptyset$.

\section{COBRA -- General Scheme}

In this section we will introduce Continuous Best-Response Approach
(COBRA), a decentralized method for trajectory coordination in multi-robot
systems. The general formulation of the algorithm assumes that each
of the robots in the system is able to compute an optimal trajectory
for itself from its current position to a given destination position
in the presence of moving obstacles without prescribing how such a
trajectory should be computed. In order to synchronize the information
flow and ensure that the robots plan their trajectory using up-to-date
information about the trajectories of others, robots are using a distributed
token-passing mechanism~\cite{ghosh2010DistributedSystems} in which
the token is used as a synchronized shared memory holding current
trajectories of all robots. We identify a token $\Phi$  with a set
$\{(a_{i},\pi_{i})\}$, which contains at most one tuple for each
robot $a=1\ldots,n$. Each such tuple represents the fact that robot
$a$ is moving along trajectory $\pi$. At any given time the token
can be held by only one of the robots and only this robot can read
and change its content. 

A robot newly added to the system tries to obtain the token and to
register itself with a trajectory that stays at its initial position
forever.After all the robots have been added to the system, the user
can start with assigning relocation tasks to individual robots.

When a new relocation task is received by robot $i$, the robot requests
the token $\Phi$. When the token is obtained, the robot runs a trajectory
planner to find a new ``best-response'' trajectory to fulfill the
relocation task. The trajectory is required a) to start at the robot's
current position $p$ at time $t_{\mathrm{now}}+t_{\mathrm{planning}}$
(at the end of the planning window), b) to reach the goal position
$g$ as soon as possible and remain at $g$ and c) to avoid collisions
with all other robots following trajectories specified in the token.
If such a trajectory is successfully found, the token is updated with
the newly generated trajectory and released so that other robots can
acquire it. Then, the robot starts following the found trajectory.
Once the robot successfully reaches the destination, it can accept
new relocation tasks. The pseudocode of the COBRA algorithm is listed
in Algorithm~\ref{alg:New-task-handling}.

\SetKwProg{alg}{Algorithm}{}{}
\SetKwProg{on}{On}{}{}
\SetKwProg{function}{Function}{}{}
\SetKwFunction{besttraji}{Best-traj$_i$}
\SetKwFunction{besttrajj}{Best-traj$_j$}
\SetKwFunction{pp}{PP}
\LinesNumbered
\begin{algorithm}[t]
\on{registered to the system at position $s$}{

	 $\Phi\leftarrow$ request token \;\label{acquire-token-registering}

	 $\pi\leftarrow$ $\pi(t)\text{ such that }\,\forall t\in[0,\infty):\;\pi(t)=s$
\;

	 $\Phi\leftarrow$ $\Phi\cup\left\{ (i,\pi)\right\} $ \;

	 release token $\Phi$ \;

}

\on{new relocation task $s\rightarrow g$ assigned}{

	 $\Phi\leftarrow$ request token \;\label{acquire-token-new-task}

	 assert $g$ is not a destination of another robot\;

	 $\Phi\leftarrow$ $\left(\Phi\setminus\left\{ (i,\pi')\,:\,(i,\pi')\in\Phi\right\} \right)$
\;

	 $\Delta\leftarrow\underset{(j,\pi_{j})\in\Phi}{\bigcup}R_{j}(\pi_{j})$\;

	 $t_{\mathrm{dep}}\leftarrow t_{\mathrm{now}}+t_{\mathrm{planning}}$
\;

	 $\pi\leftarrow$ \besttraji{$s$,$t_{\mathrm{dep}}$,$g$,$\Delta$}
\;\label{best-traj}

	 \If{$\pi=\emptyset$}{

		report failure

	}

	 $\Phi\leftarrow$ $\Phi\cup\left\{ (i,\pi)\right\} $ \;\label{token-update}

	 release token $\Phi$ \;

	 start following $\pi$ at $t_{\mathrm{dep}}$ \;

}

\function{\besttraji{$s$,$t_{s}$,$g$,$\Delta$}}{

	return trajectory $\pi$ for robot $i$ that reaches $g$ in minimal
time such that 

\quad{}a) $\pi(t_{s})=s$,

\quad{}b) $\exists t_{g}\,\forall t'\in[t_{g},\infty):\:\pi(t')=g$, 

\quad{}c) $R_{i}(\pi)\cap\Delta=0$ if it exists, 

otherwise return $\emptyset$\;

}

\caption{\label{alg:New-task-handling}COBRA -- specification for robot $i$.
The current time is denoted $t_{\mathrm{now}}$, the maximum time
that can be spent in trajectory planning is denoted $t_{\mathrm{planning}}$.}
\end{algorithm}

\subsection{Theoretical Analysis}

In this section we will derive a sufficient condition under which
is the presented mechanism complete, i.e. it guarantees that all relocation
tasks will be successfully carried out without collision. First, observe
that in general a robot may fail to find a collision-free trajectory
to its destination as illustrated in Figure~\ref{fig:Failure-of-algorithm}.
\begin{figure}
\begin{centering}
\includegraphics[width=0.6\columnwidth]{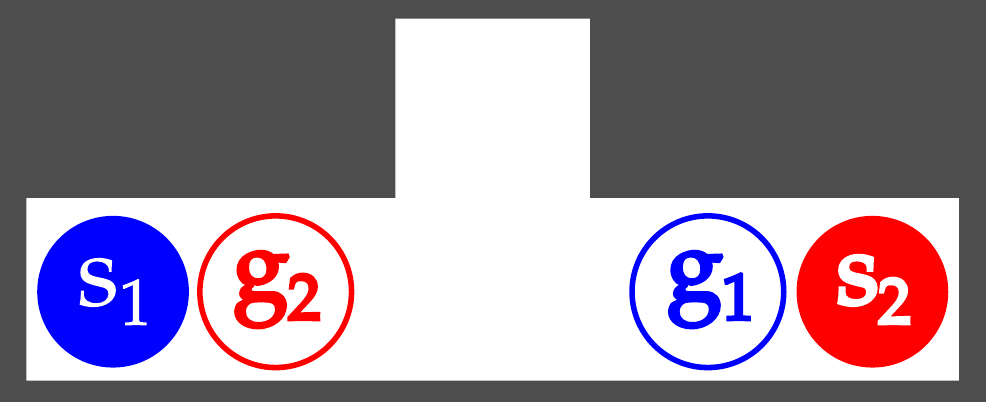}
\par\end{centering}

\caption{\label{fig:Failure-of-algorithm}Scenario in which robot~2 fails
to find a trajectory for a relocation task $s_{2}\rightarrow g_{2}$.
First, robot~1 plans a trajectory for relocation task $s_{1}\rightarrow g_{1}$.
It will travel at straight line connecting the two points at maximum
speed. Robot~2 can travel at the same maximum speed as robot~1 and
plans second from $s_{2}$ to $g_{2}$. However, it cannot reach the
exchange point in time and thus none of the available trajectories
reaches $g_{2}$ without collision with robot~1. The trajectory planning
by robot~2 fails.}
\end{figure}
 There is, however, a class of infrastructures, which we call \emph{valid
infrastructures}, where each trajectory planning is guaranteed to
succeed and consequently all relocation tasks will be carried out
without collision. 

Informally, a \emph{valid} infrastructure has its endpoints distributed
in such a way that any robot standing on an endpoint cannot completely
prevent other robots from moving between any other two endpoints.
In a valid infrastructure, a robot is always able to find a collision-free
trajectory to any other unoccupied endpoint by waiting for other robots
to reach their destination endpoint, and then by following a path
around the occupied endpoints, which is in a valid infrastructure
guaranteed to exist. 

In the following, we will describe the idea more formally. First,
let us introduce the necessary notation. Let $D(x,r)$ be a closed
disk centered at $x$ with radius $r$. Then, $\mathrm{int}_{r}\, X:=\left\{ x:D(x,r)\subseteq X\right\} $
is an $r$-interior of a set $X\subseteq\mathbb{R}^{2}$ and $\mathrm{ext}_{r}\, X:=\underset{x\in X}{\cup}D(x,r)$
is an $r$-exterior of a set $X\in\mathbb{R}^{2}$. A path is a continuous
function $p(\alpha):\,[0,1]\rightarrow\mathbb{R}^{2}$ which represents
a curve in $\mathbb{R}^{2}$. A trajectory is a function $\pi(t):\,[0,\infty)\rightarrow\mathbb{R}^{2}$
which represents a movement of a point in $\mathbb{R}^{2}$ and time.
Given a set $X\subseteq\mathbb{R}^{2}$, we say that a path $p$ is
$X$-avoiding iff $\forall\alpha:\; p(\alpha)\notin X$. Similarly,
a trajectory $\pi$ is $X$-avoiding iff $\forall t:\;\pi(t)\notin X$.
The trajectories $\pi_{i}$ and $\pi_{j}$ of robots $i$ and $j$
are said to be \emph{collision-free} iff $\forall t:\left|\pi_{i}(t)-\pi_{j}(t)\right|\geq r_{i}+r_{j}$.
A trajectory $\pi$ is $g$\emph{-terminal} iff $\exists t_{g}\,\forall t\in[t_{g},\infty):\;\pi(t)=g.$
Token $\Phi$ is called a) \emph{E-terminal} iff $\forall(a,\pi)\in\Phi:\,\pi\text{ is }g\text{-terminal}$
and $g\in E$, and b) \emph{collision-free} iff $\forall(a,\pi),\,(a',\pi')\in\Phi:\, a\neq a'\Rightarrow\pi\text{ and }\pi'\text{ are collision-free}.$
\begin{defn}
An infrastructure $(\mathcal{W},E)$ is called \emph{valid} for circular
robots having body radii $r_{1},\ldots,r_{n}$ if any two endpoints
$a,b\in E$ can be connected by a path in workspace $\mathrm{int}_{\overline{r}}\left(\mathcal{W}\setminus\underset{e\in E\setminus\{a,b\}}{\cup}D(e,\overline{r})\right)$,
where $\overline{r}=\max\{r_{1},\ldots,r_{n}\}$. 
\end{defn}
In other words, there must exists a path between any two endpoints
with at least $\overline{r}$-clearance with respect to the static
obstacles and at least $2\overline{r}$-clearance to any other endpoint.
Figure~\ref{fig:infrastructure} illustrates the concept of a valid
infrastructure.

\begin{figure}
\begin{centering}
\subfloat[Valid infrastructure: The workspace $\mathcal{W}$ and endpoints $\{e_{1},e_{2},e_{3},e_{4}\}$
for robots having radius $r$ form a valid infrastructure. ]{\centering{}\includegraphics[width=0.4\columnwidth]{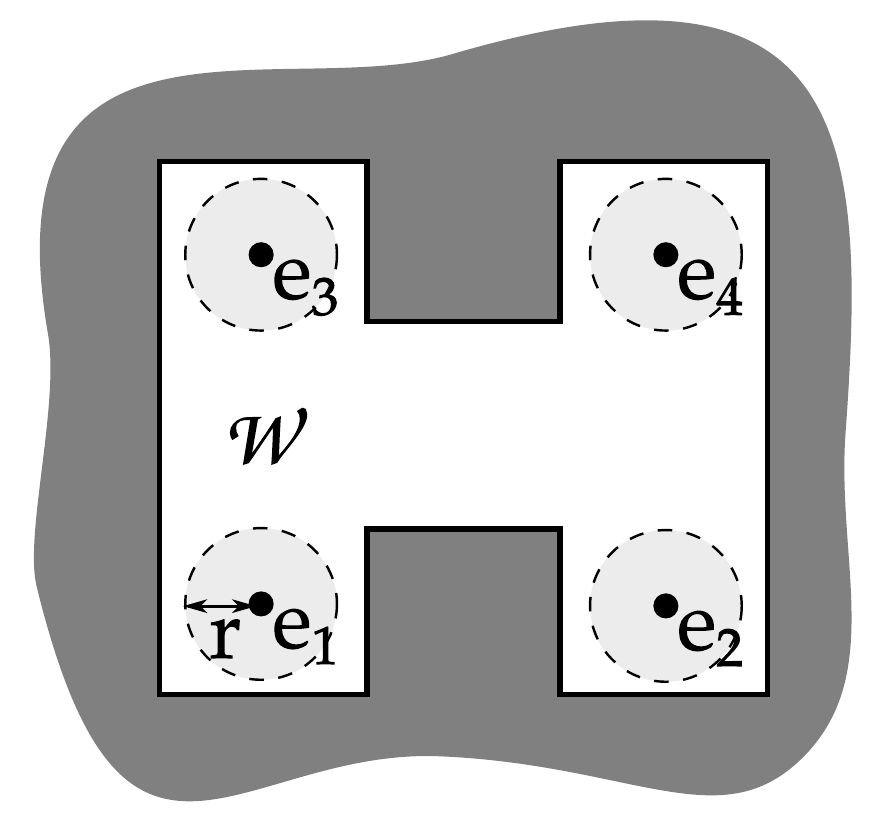}}~~~~\subfloat[Non-valid Infrastructure: The workspace $\mathcal{W}$ and endpoints
$\{e_{1},e_{2},e_{3}\}$ do not form a valid infrastructure because
there is no path from $e_{1}$ to $e_{2}$ with $2r$-clearance to
$e_{3}$ for a robot having radius~$r$.]{\begin{centering}
\includegraphics[width=0.4\columnwidth]{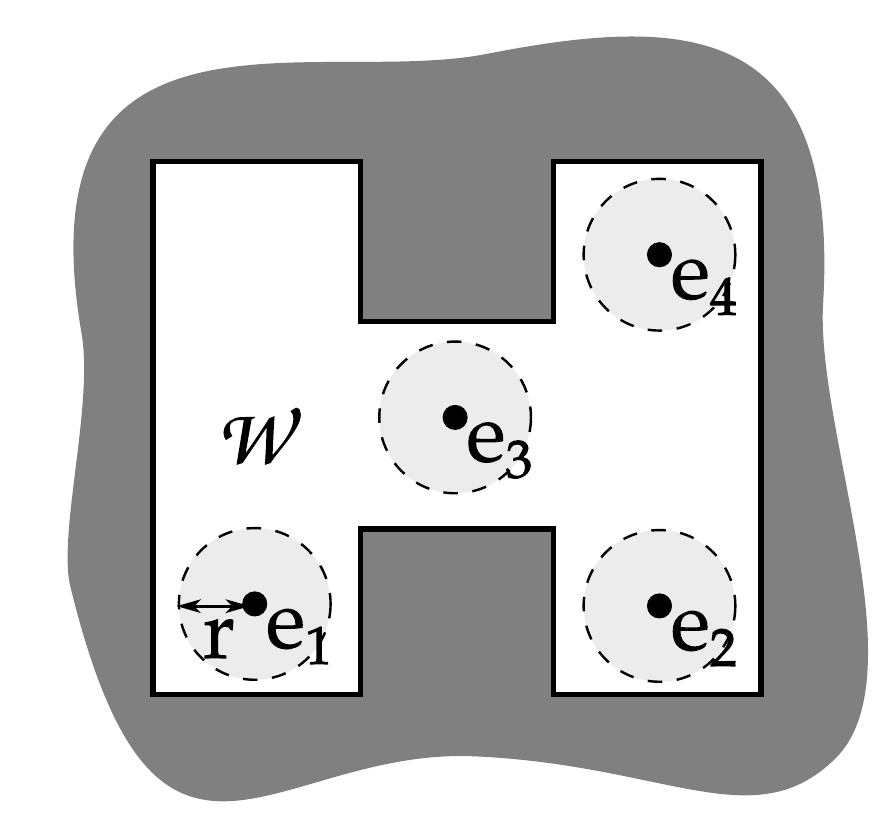}
\par\end{centering}

}
\par\end{centering}

\caption{\label{fig:infrastructure}Valid and non-valid infrastructure }
\end{figure}

The notion of valid infrastructures follows the structure typically
witnessed in man-made environments that are intuitively designed to
allow efficient transit of multiple people or vehicles. In such environments,
the endpoint locations where people or vehicles are stopping for longer
time are separated from the transit area that is reserved for travel
between these locations. 

If we take the road network as an example, the endpoints would be
the parking places and the system of roads is built in such a way
that any two parking places are reachable without crossing any other
parking place. Similar structure can be witnessed in offices and factories.
The endpoints would be all locations, where people may need to spend
longer periods of time, e.g. surroundings of the work desks or machines.
As we know from our every day experience, work desks and machines
are typically given enough free room around them so that a person
working at a desk or a machine does not obstruct people moving between
other desks or machines. We can see that real-world environments are
indeed often designed as valid infrastructures.

\subsubsection*{Completeness in Valid Infrastructures}

In this section, we show that if robots are operating in a valid infrastructure
and use COBRA for trajectory coordination, they will successfully
carry out all assigned relocation tasks without collisions. Our analysis
assumes the following setup of the multi-robot system. First, the
robots must operate in a valid infrastructure $(\mathcal{W},E)$.
When the system is initialized, robots are located at distinct initial
endpoints of the infrastructure. After the initialization is finished,
robots start accepting relocation tasks from some higher-level component,
which ensures that each destination is an endpoint of the infrastructure
and two robots will never have simultaneously assigned relocation
task with the same destination. Each robot uses a complete trajectory
planner and any trajectory generated for a robot, will be precisely
followed by robot.

\begin{prop}
\label{prop:Planning-doesnt-fail}If token $\Phi$ acquired during
handling of a new task $s\rightarrow g$ assigned to robot $i$ (on
line~\ref{acquire-token-new-task}~in Algorithm~\ref{alg:New-task-handling})
is E-terminal and collision-free, then the subsequent trajectory planning
succeeds and returns a trajectory that is $g$-terminal and collision-free
with respect to other trajectories in $\Phi$. \end{prop}
\begin{proof}
Because $(\mathcal{W},E)$ is a valid infrastructure and $s,g\in E$,
there exists a $\mathrm{ext}_{r}\,\left(E\setminus\left\{ s,g\right\} \right)$-avoiding
path $p$ going from $s$ to $g$. All trajectories in $\Phi$ are
E-terminal, which implies that eventually they reach one of the endpoints
and stay at that endpoint. Consequently, there exists a time point
$\overline{t}$ after which all the robots reach their destination
endpoints and stay there. A $g$-terminal and collision-free trajectory
for robot $i$ can be constructed as follows:
\begin{itemize}
\item In time interval $t\in[t_{s},\max(t_{s},\overline{t})]:\:\pi(t)=s$.
The trajectory cannot be in collision with any other trajectory in
$\Phi$ during this interval: From the assumption that new relocation
tasks can be assigned only after the previous relocation task has
been completed, we have $\forall t\in[t_{now},\infty):\,\pi(t)=g'$.
From the assumption that the start position of a new relocation task
$s$ must be the same as the goal of the robots previous task $g'$,
we know $\forall t>t_{now}:\,\pi(t)=s$. Since $\Phi$ is collision-free,
all trajectories of robots other than $j$ from $\Phi$ must be avoiding
$s$ with $r_{i}+r_{j}$ clearance, where $r_{j}$ is the radius of
the other robot in $\Phi$. Therefore, in time interval $[t_{s},\max(t_{s},\overline{t})]$,
robot $i$ will be collision-free with respect to other trajectories
stored in $\Phi$.
\item In time interval $t\in[\overline{t},\infty]:\:$$\pi$ follows path
$p$ until the goal position g is reached. The trajectory cannot be
in collision with any trajectory in $\Phi$ during this interval:
After time $\overline{t}$, all robots are at their respective goal
positions $G$, which satisfy: a) $G\subseteq E$, b) $s\notin G$,
otherwise some of the trajectories would have to be in collision with
$\pi$, which contradicts the finding from the previous point, and
c) $g\notin G$ because we assume that two robots cannot have simultaneously
relocation tasks with the same destination. We know that the path
$p$ avoids regions $\mathrm{ext}_{2\overline{r}}\,\left(E\setminus\left\{ s,g\right\} \right)$
and thus the trajectory cannot be in collision with any other trajectory
in $\Phi$ during the interval $[\overline{t},\infty)$.
\end{itemize}
Such a trajectory can always be constructed and thus any single-robot
complete trajectory planning method must find it.
\end{proof}

\begin{prop}
\label{prop:Phi-is-satisfactory-and-collision-free}Every time $\Phi$
is acquired by a robot during new relocation task handling (line~\ref{acquire-token-new-task}~in~Algorithm~\ref{alg:New-task-handling}),
$\Phi$ is E-terminal and collision-free. \end{prop}
\begin{proof}
By induction on $\Phi$. Take an arbitrary sequence of individual
robots acquiring the token $\Phi$ and updating it. Let the induction
assumption be: Whenever $\Phi$ is acquired by a robot to handle a
new relocation task, $\Phi$ is E-terminal and collision-free. 

Base case: When the last robot registers itself, all robots have a
trajectory in $\Phi$ that remains at their start positions forever.
Since the start position are assumed to be distinct and at least $2\overline{r}$
apart, then the $\Phi$ is collision-avoiding. Further, $\Phi$ is
trivially E-terminal because after the last robot registers, all robots
stay at their initial position, which is an endpoint. 

Inductive step: From the inductive assumption, we know that $\Phi$
acquired at line~\ref{acquire-token-new-task}~in~Algorithm~1
is E-terminal and collision-free. Using Proposition~\ref{prop:Planning-doesnt-fail},
we see that for any relocation task $s\rightarrow g$, the algorithm
will find a trajectory $\pi$ that is g-terminal and collision-free
with respect to other trajectories in $\Phi$. From our assumption,
the destination $g\in E$. By adding trajectory $\pi$ to $\Phi$
at line~\ref{token-update}~in~Algorithm~\ref{alg:New-task-handling},
the token $\Phi$ remains E-terminal and collision-free. 
\end{proof}

\begin{thm}
\label{thm:General_Cobra_is_complete} If a team of robots operates
in a valid infrastructure $(\mathcal{W},E)$ and the COBRA algorithm
is used to coordinate relocation tasks between the endpoints of the
infrastructure, then all relocation tasks will be successfully carried
out without collision. \end{thm}
\begin{proof}
By contradiction. Assume that a) there can be a relocation task $s\rightarrow g$
assigned to robot $i$ that is not carried out successfully or b)
two robot collide at some time point during the operation of the system. 

Case A: A relocation task $s\rightarrow g$ assigned to a robot $i$
has not been successfully completed. We assume that robots are able
to perfectly follow any given trajectory $\pi$. Therefore either
$\pi$ is not $g$-terminal or robot $i$ collided with another robot
during the execution of the trajectory. From Proposition~\ref{prop:Planning-doesnt-fail}~and~\ref{prop:Phi-is-satisfactory-and-collision-free}
we know that the \besttraji routine will always return a $g$-satisfactory
trajectory. The possibility that the robot collided while carrying
out a relocation task is covered by Case B and is shown to be impossible.
Thus, the relocation task $s\rightarrow g$ assigned to robot $i$
will be completed successfully.

Case B: Robots $i$ and $j$ collide. Since we assume perfect execution
of trajectories, it must be the case that there is $\pi_{i}$ and
$\pi_{j}$ that are not collision-free. Since the collision relation
is symmetrical, w.l.o.g., it can be assumed that $\pi_{j}$ was added
to $\Phi$ later than $\pi_{i}$. This implies that $\pi_{i}$ was
in the token $\Phi$ when $\pi_{j}$ was generated within \besttrajj
routine. However, this would imply that the trajectory planning returned
a trajectory that is not collision-free with respect to trajectories
in $\Phi$, which is a contradiction with constraints used during
the trajectory planning. Thus, robots $i$ and $j$ do not collide.

We can see that neither case A nor case B can be achieved and thus
all relocation tasks are carried out without any collision.
\end{proof}

\section{COBRA with a Time-extended Roadmap Planner}

In the previous section we have presented the general scheme of COBRA
that assumes that every robot is able to find an optimal best-response
trajectory for itself using an arbitrary complete algorithm. In practice,
such a planning would be often done via graph search in a discretized
representations of the static workspace. In this section, we will
analyze the properties of COBRA when all robots use a roadmap-based
planner to plan their best-response trajectory.

The function \besttraji{$s$,$t_{s}$,$g$,$\Delta$} used on line~\ref{best-traj}
in Algorithm~\ref{alg:New-task-handling} returns an optimal trajectory
for a particular robot $i$ from $s$ to $g$ starting at time $t_{s}$
that avoids space-time regions $\Delta$ occupied by other robots.
We will now assume that the robots use a time-extended roadmap planner
to compute such a trajectory. The planner takes a graph representation
of the static workspace and extends it with a discretized time dimension.
The resulting time-extended roadmap is subsequently searched using
Dijkstra's shortest-path algorithm (possibly with some admissible
heuristic). The time-extended roadmap planner can be implemented as
follows. Let us have a ``roadmap'' graph $G=(V,L)$ representing
an arbitrary discretization of the static workspace $\mathcal{W}$,
where $V\subseteq\mathcal{W}$ represent the discretized positions
from $\mathcal{W}$ and $L\subseteq V\times V$ represents possible
straight-line transitions between the vertices. The time-extended
graph $\overline{G}=(\overline{V},\overline{L})$ for robot $i$ with
time step $\delta t$ starting at position $s$ at time $t_{s}$ is
recursively defined as follows:

\begin{gather}
(s,t_{s})\notin\Delta\Rightarrow(s,t_{s})\in\overline{V}\\
\text{and}\nonumber \\
\forall(v,t)\in V\,\wedge\,(v_{1},v_{2})\in L:\,\nonumber \\
t'=\left\lceil \frac{\left|v_{2}-v_{1}\right|}{\delta t\cdot v_{i}^{max}}\right\rceil \delta t\,\wedge\mathrm{line}\left((v_{1},t),(v_{2},t')\right)\cap\Delta=\emptyset\\
\Rightarrow\nonumber \\
(v_{2},t')\in\overline{V}\,\wedge\,\left((v_{1},t),(v_{2},t')\right)\in\overline{L},\nonumber 
\end{gather}
where $\left\lceil \cdot\right\rceil $ represents ceiling, $\left|\cdot\right|$
is the Euclidean norm, and $\mathrm{line(x,y)}$ represents the set
of points lying on the line $x\rightarrow y$. It is important to
realize that because we force each space-time edge to end a time that
is a multiple of $\delta$t, some of the edges may in fact be traveled
at a slower speed than $v_{i}^{\mathrm{max}}$. The best-response
trajectory is then constructed by finding the shortest path in $\overline{G}$
starting from the vertex $(s,t_{s})$ to a goal vertex that satisfies
$(g,t_{g})\,\wedge\,\mathrm{line\left((g,t_{g}),(g,\infty)\right)\cap\Delta=\emptyset}$.

\subsection*{Valid Infrastructure Roadmap}

The notion of a valid infrastructure can be extended to discretized
representations of the robots' common workspace as follows.
\begin{defn}
A graph $G=(V,L)$ is a roadmap for a valid infrastructure $(\mathcal{W},E)$
and robots with radii $r_{1},\ldots r_{n}$ if $E\subseteq V$ and
any two different endpoints $a,b\in E$ can be connected by a path
$p=l_{1},\ldots,l_{k}$ in graph $G$ such that 
\[
\underset{i=1,\ldots,k}{\forall}\mathrm{line(}l_{i})\subseteq\mathrm{int}_{\overline{r}}\left(\mathcal{W}\setminus\underset{e\in E\setminus\{a,b\}}{\cup}D(e,\overline{r})\right),
\]
 where $\overline{r}=\max\{r_{1},\ldots,r_{n}\}$. 
\end{defn}
Analogically to the valid infrastructure, we require that the roadmap
contains a path that connects any two endpoints with at least $\mbox{\ensuremath{\overline{r}}}$
clearance to static obstacles and $2\overline{r}$ clearance to other
endpoints.

\subsection*{Theoretical Analysis}

In this section, we will show that COBRA with a time-extended roadmap
planner operating on a roadmap for valid infrastructure is complete
and the trajectory for any relocation task will be computed in time
quadratic to the number of robots present in the system. In the following,
we will assume that $G=(V,L)$ is a roadmap for a valid infrastructure
$(\mathcal{W},E)$. 
\begin{thm}
If $G=(V,L)$ is a roadmap for a valid infrastructure $(\mathcal{W},E)$
and COBRA with a time-extended roadmap planner is used to coordinate
relocation tasks between the endpoints of the infrastructure, then
all relocation tasks will be successfully carried out without collision.\end{thm}
\begin{proof}
The line of argumentation used to prove the completeness of COBRA
with an arbitrary complete trajectory planner (Theorem~\ref{thm:General_Cobra_is_complete})
is also applicable for COBRA with time-extended roadmap planner (TERP)
-- we only need to show that Proposition~\ref{prop:Planning-doesnt-fail}
holds when TERP is used to find the best-response trajectory. The
original argument can be adapted as follows: If the robot acquires
an E-terminal token, there is a time-point $\overline{t}$, when all
robots in $\Phi$ reach their destinations and stay there. A best
response trajectory for a relocation task $s\rightarrow g$ can always
be constructed by waiting at the robot's start endpoint until $\left\lceil \frac{\overline{t}}{\delta t}\right\rceil \delta t$
(the first discrete time-point, when all robots from $\Phi$ are on
an endpoint) and then by following the shortest path from $s$ to
$g$ on roadmap $G$. 
\end{proof}

\subsection*{Complexity}

In this section we derive the computational complexity of the COBRA
algorithm when the time-extended roadmap planner is used for finding
the best-response trajectory. We will focus on the complexity of trajectory
planning for a single relocation task and show that if a robot acquires
a token with trajectories of other robots, all the robots in the token
will reach their respective destination endpoints in a relative time
that is bounded by a factor linearly dependent on the number of robots.
Therefore, the state space that needs to be searched during the trajectory
planning is linear in the number of robots.

\paragraph*{Assumptions and Notation: \textmd{Let $f(\pi)$ denote the time when
trajectory $\pi$ reaches its destination, i.e. $f(\pi):=\min\, t\,\text{s.t.}\, t'\geq t:\,\pi(t')=g,\: g\in E$.
Further, let $A(\Phi,t)$ denote the set of ``active'' trajectories
from token $\Phi$ at time $t$ defined as $A(\Phi,t):=\left\{ \pi:(\cdot,\pi)\in\Phi\,\text{s.t. }f(\pi)\geq t\right\} $.
The latest time when all trajectories in token $\Phi$ reach their
destination endpoints is given by function $F(\Phi):=\protect\underset{(\cdot,\pi)\in\Phi}{\max}f(\pi)$.
Further, let $r$ denote the duration of the longest possible individually
optimal trajectory between two endpoints by any of the robots: 
\[
r=\protect\underset{e_{1},e_{2}\in E}{\max}\frac{\text{length of shortest path from }e_{1}\text{ to }e_{2}\text{ in }G}{\protect\underset{i=1\ldots n}{\min}v_{i}^{max}}.
\]
}}
\begin{lem}
\label{lem:final_time_of_active_is_bounded}At any time point $t$
we have $F(A(\Phi,t))\leq t+\left|A(\Phi,t)\right|r$.\end{lem}
\begin{proof}
Observe that the token only changes during new task handling. Let
us consider an arbitrary sequence of tasks being handled by different
robots at time-point $\tau_{1},\tau_{2}\ldots$. Initially, at time
$t=0$ the token is empty and the inequality holds trivially: $F(\emptyset)\leq0$.
We will inductively show that the inequality holds after the token
update during each such task handling, which implies that it also
holds for the time period until the next update of the token. Now,
let us take $k^{\mathrm{th}}$-task handling by robot $i$ with current
trajectory $\pi_{i}$, that at time $\tau^{k}$ obtains token $\Phi^{k-1}$.
By induction hypothesis we have $F(A(\Phi^{k-1},\tau^{k}))\leq\tau^{k}+\left|A(\Phi^{k-1},\tau^{k})\right|r$.
We know that $\pi_{i}\notin A(\Phi^{k-1},\tau^{k})$, because the
robot can only accept new tasks after it has finished the previous
one and thus $f(\pi_{i})<\tau^{k}$. Further, we know that 1) $\forall\pi\in\Phi^{k-1}\setminus A(\Phi^{k-1},\tau^{k})\;\forall t>\tau^{k}:\:\pi(t)\in E$
and 2) $\forall\pi\in A(\Phi^{k-1},\tau^{k})\;\forall t>F(A(\Phi^{k-1},\tau^{k}))\:\pi(t)\in E$,
i.e. ``inactive'' trajectories are on the endpoints immediately,
while active trajectories will reach their endpoints and stay there
after $F(A(\Phi^{k-1},\tau^{k})$. Then, the robot finds its new trajectory
$\pi_{i}^{\star}$, which can be in the worst-case constructed by
waiting at the current endpoint and then following the shortest endpoint-avoiding
path (which is in infrastructures guaranteed to exist and its duration
is bounded by $r$) to the destination endpoint. Such a path reaches
the destination in $\tau^{k}\leq f(\pi_{i}^{\star})\leq F(A(\Phi^{k-1},\tau^{k})+r$.
Then, the robot updates token $\Phi^{k}\leftarrow\Phi^{k-1}\setminus\left\{ \pi_{i}\right\} \cup\left\{ \pi_{i}^{\star}\right\} $.
We know that $\pi_{i}\notin A(\Phi^{k-1},\tau^{k})$, but $\pi_{i}^{\star}\in A(\Phi^{k},\tau^{k})$
and thus $\left|A(\Phi^{k},\tau^{k})\right|=\left|A(\Phi^{k-1},\tau^{k})\right|+1$.
By rearrangement 
\[
\begin{aligned}F(A(\Phi^{k-1},\tau^{k})) & \leq\tau^{k}+\left|A(\Phi^{k-1},\tau^{k})\right|r\\
F(A(\Phi^{k},\tau^{k})) & \leq F(A(\Phi^{k-1},\tau^{k}))+r\\
F(A(\Phi^{k},\tau^{k})) & \leq\tau^{k}+\left|A(\Phi^{k-1},\tau^{k})\right|r+r\\
F(A(\Phi^{k},\tau^{k})) & \leq\tau^{k}+\left|A(\Phi^{k-1},\tau^{k})+1\right|r\\
F(A(\Phi^{k},\tau^{k})) & \leq\tau^{k}+\left|A(\Phi^{k},\tau^{k})\right|r
\end{aligned}
.
\]
\end{proof}
\begin{lem}
\label{lem:Best_response_trajectory_duration_is_limited}During each
relocation task handling of robot $i$ at time $t$, there is a trajectory
$\pi$ that reaches the destination of the relocation task in time
$f(\pi)\leq t+nr$. \end{lem}
\begin{proof}
It is known that $F(A(\Phi,t))\leq t+\left|A(\Phi,t)\right|r$. Token
$\Phi$ is updated in such a way that it contains at most one record
for each robot. Assume that robot $i$ handles a new relocation task.
Before planning, robot $i$ removes its trajectory from token $\Phi$
and thus we have $\left|\Phi\right|\leq n-1$. Since $\left|A(\Phi,t)\right|\subseteq\Phi$,
we have $\left|A(\Phi,t)\right|\leq n-1$ and using Lemma~\ref{lem:final_time_of_active_is_bounded},
we get $F(A(\Phi,t))\leq t+(n-1)r$, i.e. all other robots will be
at their respective destination endpoint at latest time $t+(n-1)r$.
In the worst case, trajectory $\pi$ can be constructed by waiting
at robot's $i$ current endpoint until time $t+(n-1)r$ and then following
the shortest path to its destination endpoint, which can take at most
$r$. Trajectory $\pi$ thus reaches the destination endpoint latest
at time $t+(n-1)r+r=t+nr$.\end{proof}
\begin{thm*}
The worst-case asymptotic complexity of a single relocation task handling
using COBRA with time-extended roadmap planning is $O(n^{2}v^{2}(\frac{1}{\delta t})^{2}r(d+r))$,
where $n$ is the number of robots in the system, $v$ is the number
of vertices in the roadmap graph, $\delta t$ is the time-discretization
step, $r$ is the maximum duration of a single relocation when the
interactions between robot are ignored, and $d$ is the duration of
longest space-time edge in time-extended roadmap.\end{thm*}
\begin{proof}
Suppose that robot $i$ handles a relocation task $s\rightarrow g$.
Let $v$ denote the number of vertices of the roadmap graph. In the
worst case, the time-extended graph can contain all vertices from
the roadmap $G$ for each time step. The best-response trajectory
is the shortest path from the initial vertex $(s,t_{s})$, where $t_{s}=t+t_{\mathrm{planning}}$
to a goal vertex $(g,t_{g})$. The search algorithm will only need
to examine the subgraph for the time interval $[t_{s},t_{g}]$, which
contains at most $\left\lceil \frac{(t_{g}-t_{s})}{\delta t}\right\rceil v$
vertices. We know $t_{g}\leq t+nr$ (from Lemma~\ref{lem:Best_response_trajectory_duration_is_limited})
and $t_{s}>t$ (because $t_{s}=t+t_{\mathrm{planning}}$ and $t_{\mathrm{planning}}>0$)
and thus this subgraph will have at most $\left\lceil \frac{nr}{\delta t}\right\rceil v\overset{\sim}{=}\frac{nrv}{\delta t}$
vertices for $\delta t\ll nr$. This graph first needs to be constructed
and then searched. Construction: During the construction of the space-time
subgraph for robot $i$, each edge $\epsilon$ has to be checked for
collisions with moving obstacles $\Delta$ composed of $n$ space-time
regions, each representing the disc body of another robot $j$ moving
along trajectory $\pi_{j}$ (itself composed of line segments). Deciding
whether $\epsilon$ collides with $R_{j}(\pi_{j})$ can be done in
time linear to the number of time steps edge $\epsilon$ spans, since
for each time step $\tau$, a sub-segment corresponding to time step
$\tau$ can be extracted both from $\epsilon$ and $\pi_{j}$ and
the collision-free property $\forall t:\:\epsilon(t)-\pi_{j}(t)\leq r_{j}+r_{i}$
can be validated by solving the corresponding quadratic equation.
One edge can be checked in time $O(\frac{d}{\delta t}n)$, where $d$
is the duration of the edge. There is at most $\frac{nr}{\delta t}v^{2}$
edges in the space-time subgraph and thus it can be constructed in
time $O(n^{2}v^{2}(\frac{1}{\delta t})^{2}dr)$. Search: The worst-case
time-complexity of Dijkstra's shortest path algorithm is $O(N^{2})$~,
where $N$ is the number vertices of the searched graph, which is
in our case $N=\frac{nrv}{\delta t}$. The time-complexity of search
is therefore $O\left(\left(\frac{nrv}{\delta t}\right)^{2}\right)=O(n^{2}v^{2}(\frac{1}{\delta t})^{2}r^{2})$.
By combining construction and search, we get $O(n^{2}v^{2}(\frac{1}{\delta t})^{2}r(d+r))$. 
\end{proof}

\section{Empirical Analysis}

\begin{figure}[t]
\begin{centering}
\includegraphics[width=0.75\columnwidth]{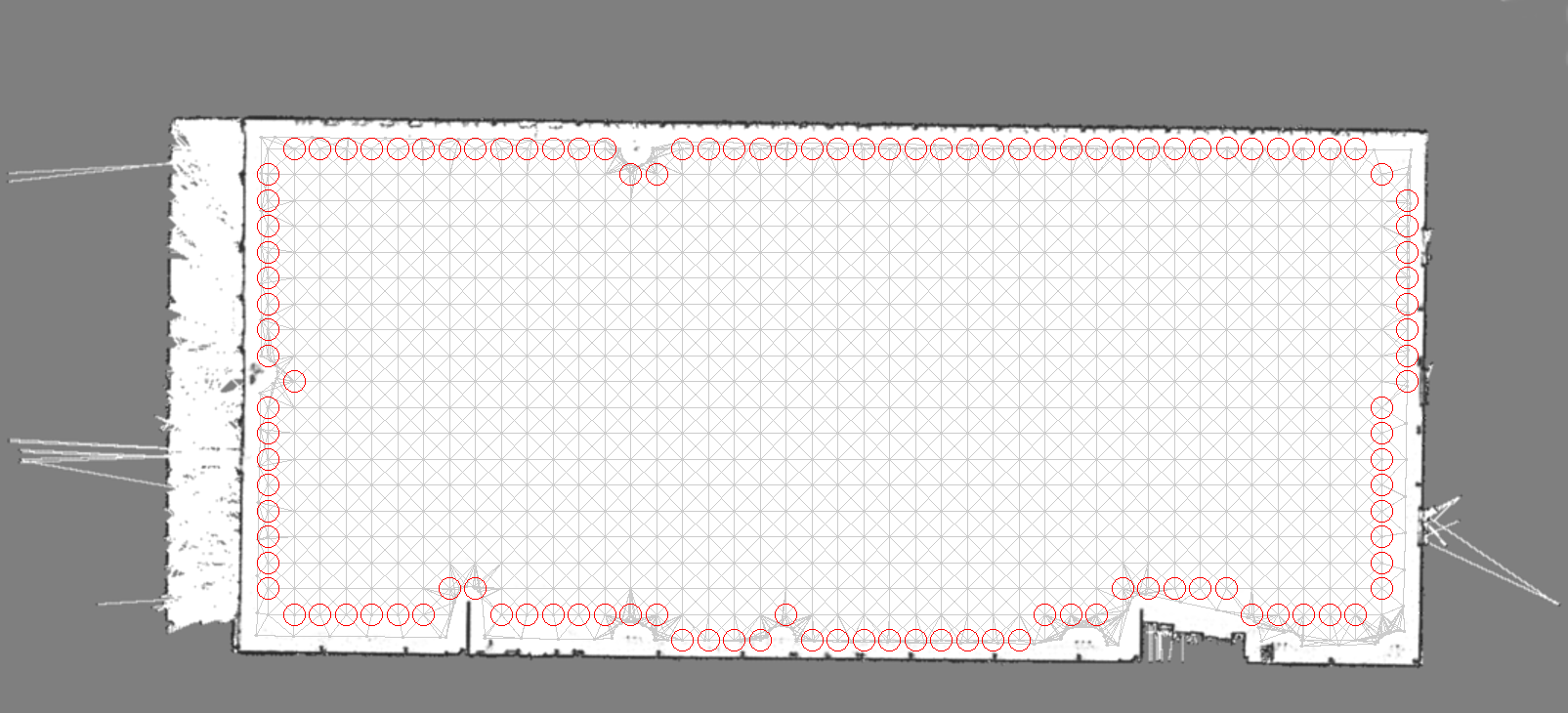}
\par\end{centering}

\begin{centering}
Hall
\par\end{centering}

\medskip{}

\begin{centering}
\includegraphics[width=0.75\columnwidth]{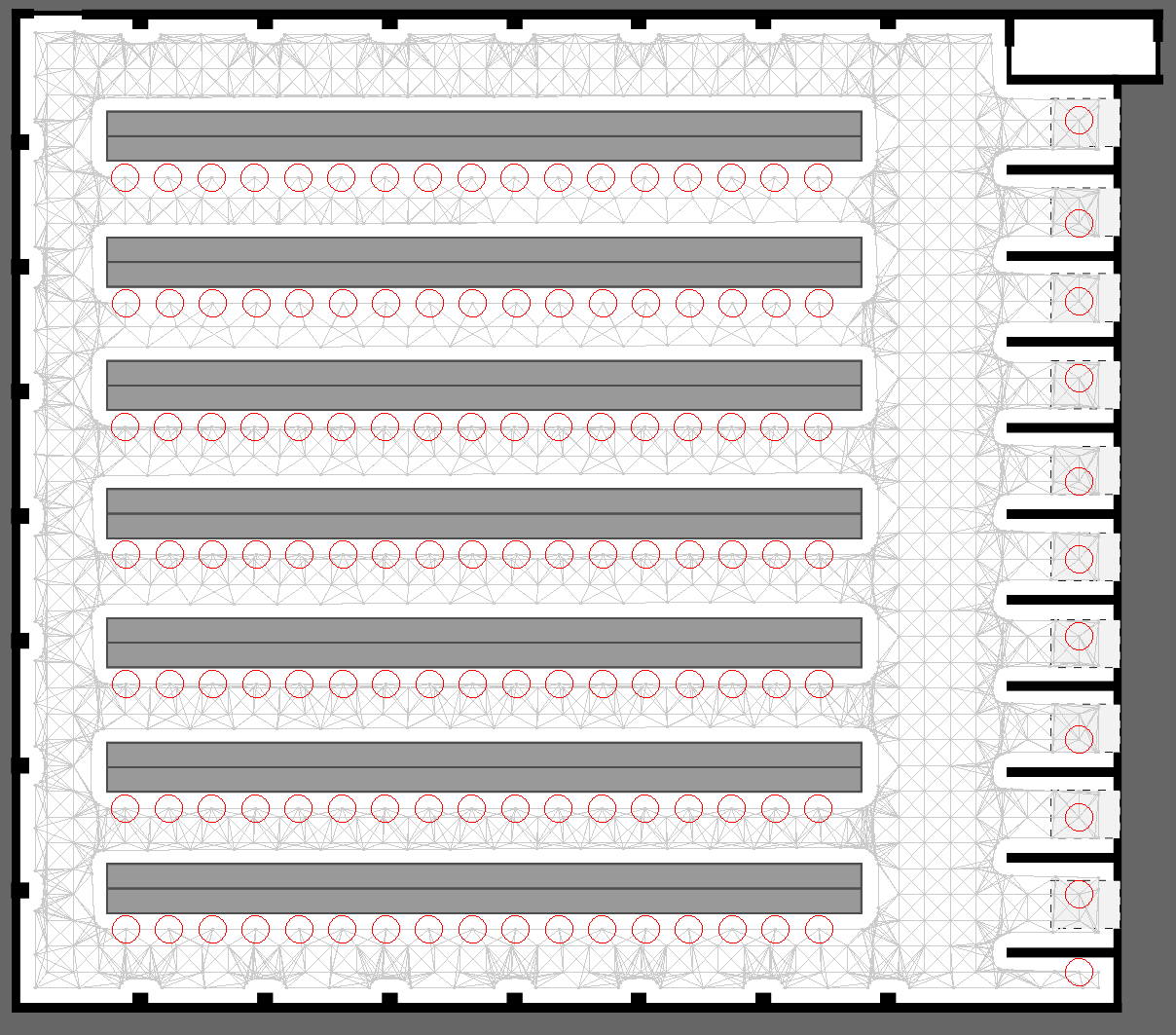}
\par\end{centering}

\begin{centering}
Warehouse
\par\end{centering}

\medskip{}

\begin{centering}
\includegraphics[width=1\columnwidth]{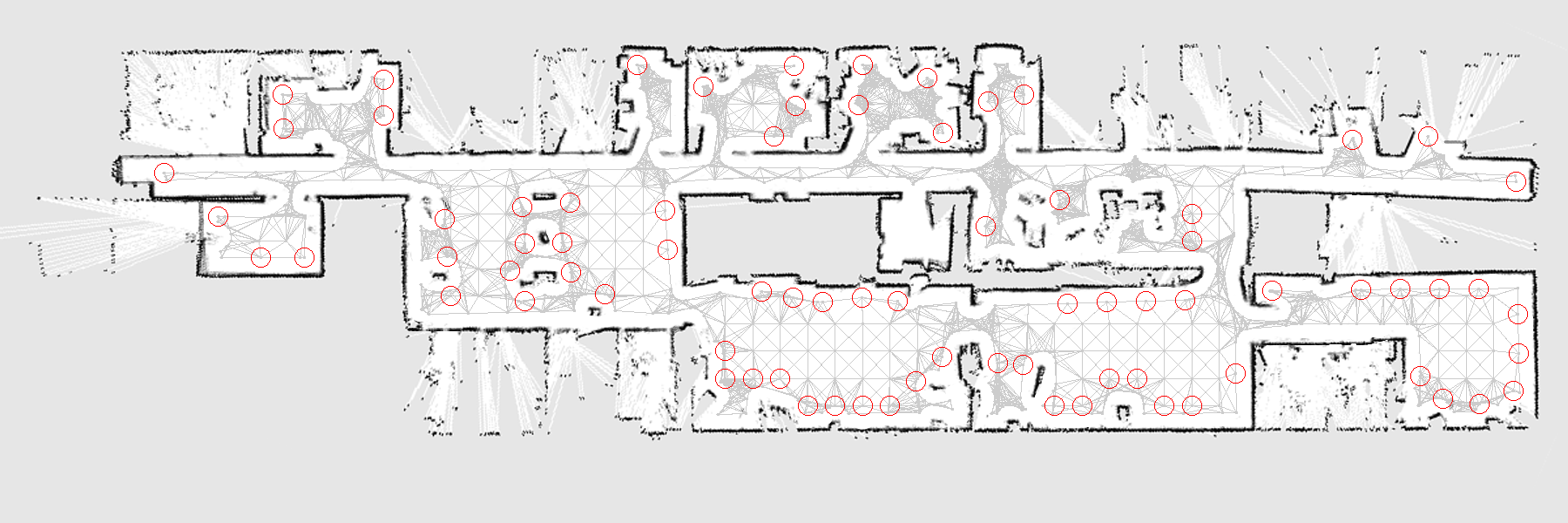}
\par\end{centering}

\begin{centering}
Office
\par\end{centering}

\medskip{}

\caption{\label{fig:Test-environments}Test environments. The infrastructure
endpoints depicted in red, the roadmap graph shown in gray. }
\end{figure}
\begin{figure*}[t]
\begin{centering}
\begin{tabular}{ccc}
Hall & Warehouse & Office\tabularnewline
\includegraphics[width=0.25\paperwidth]{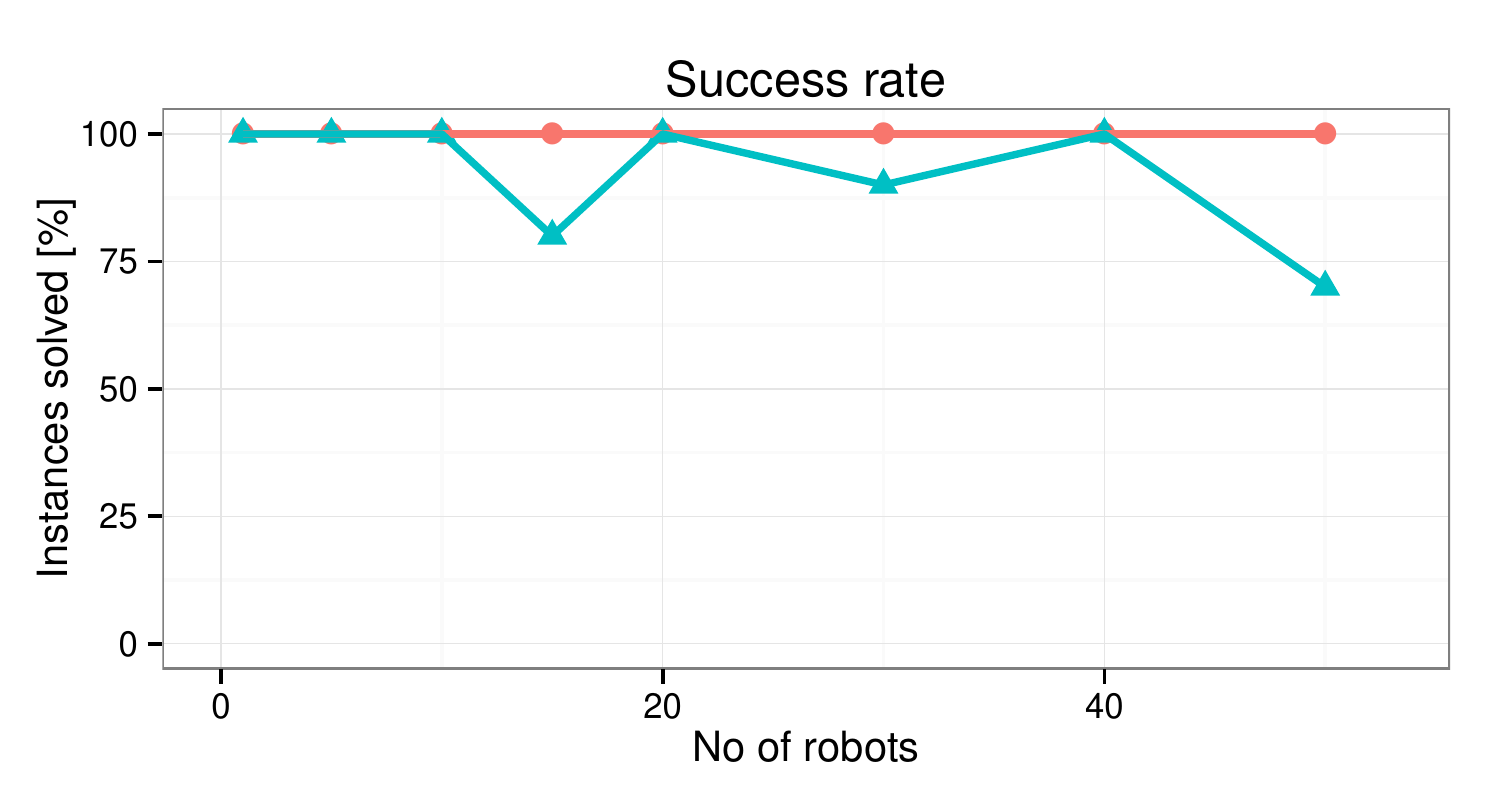} & \includegraphics[width=0.27\paperwidth]{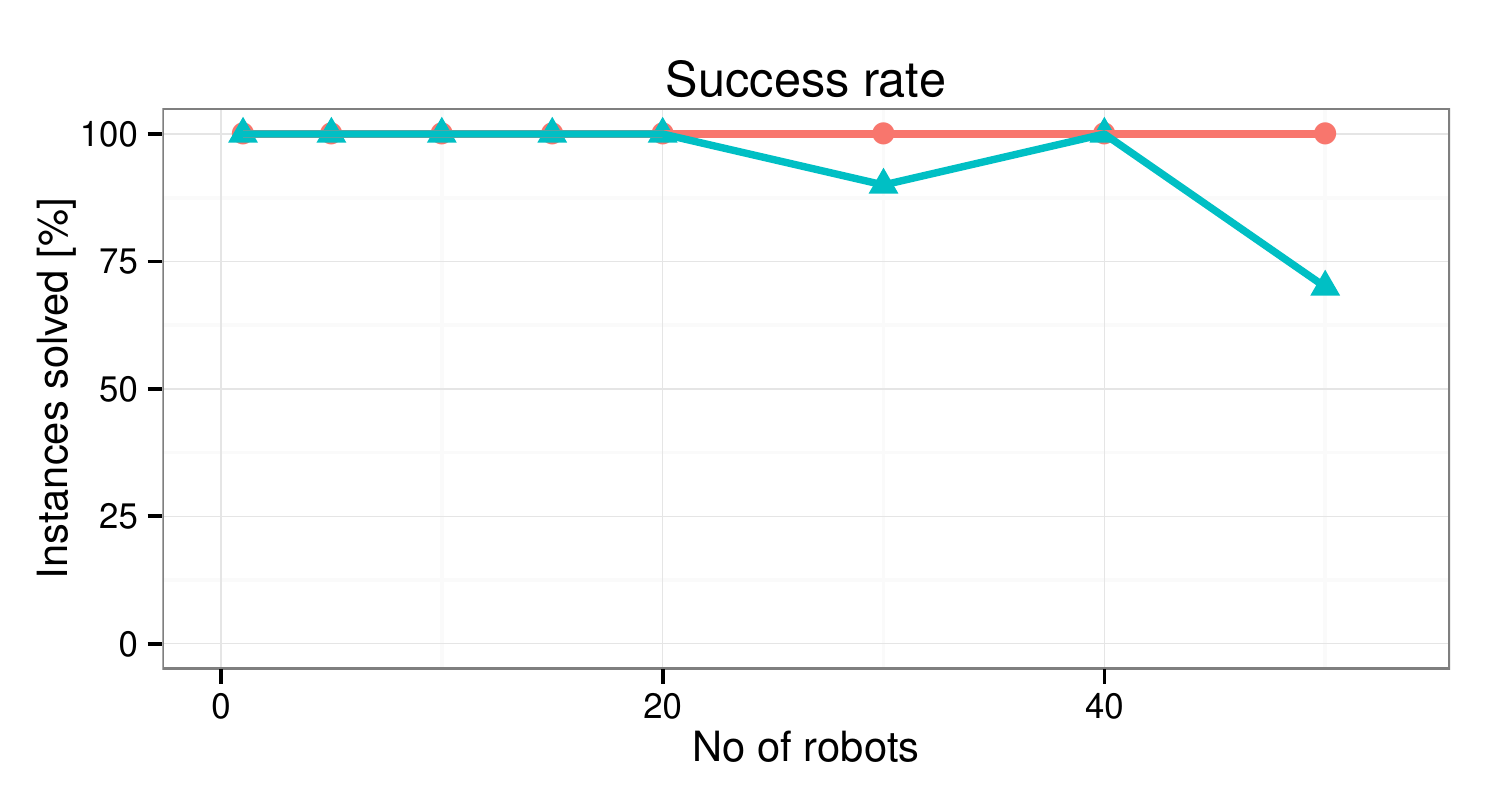} & \includegraphics[width=0.27\paperwidth]{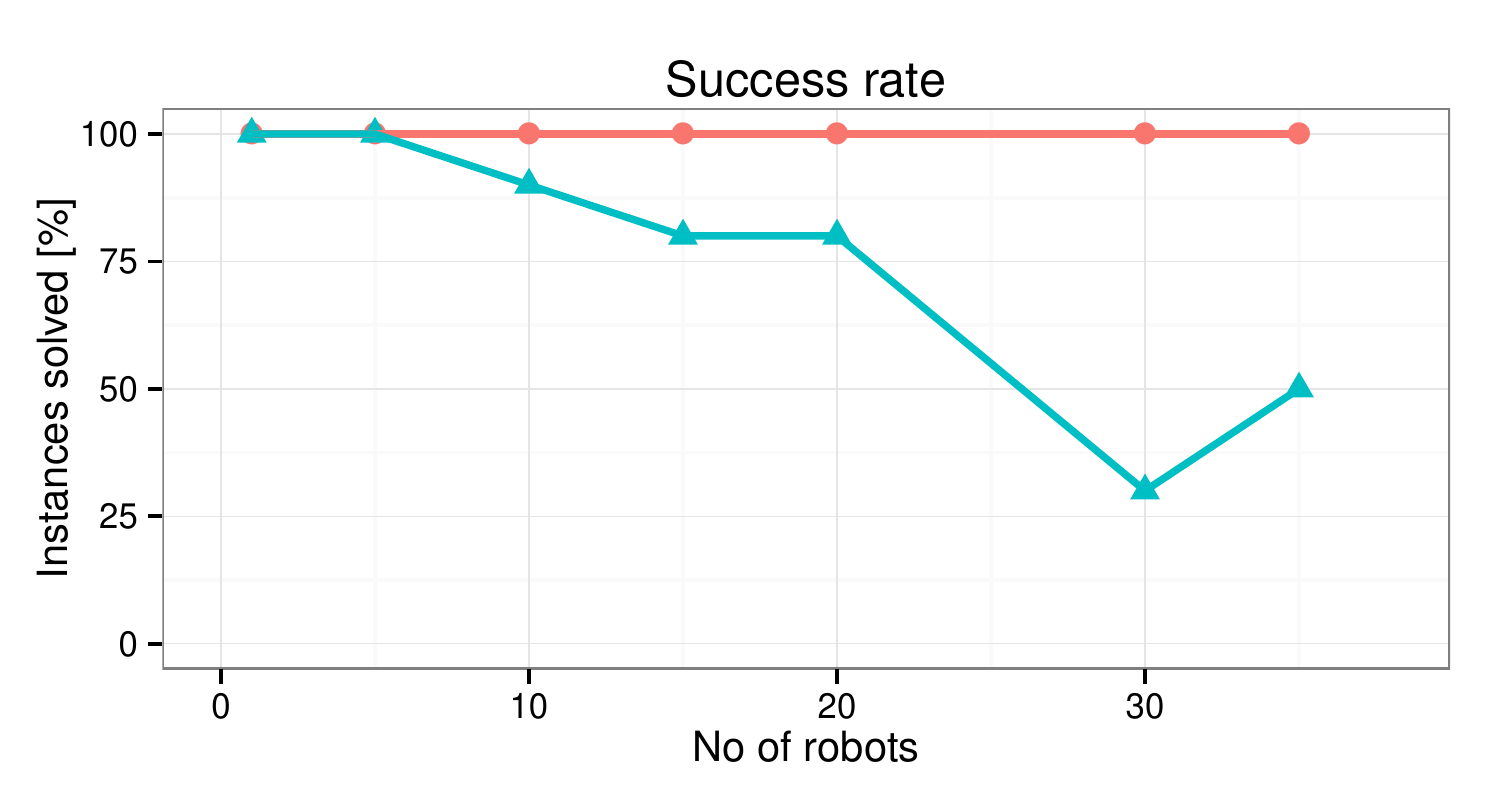}\tabularnewline
\includegraphics[width=0.25\paperwidth]{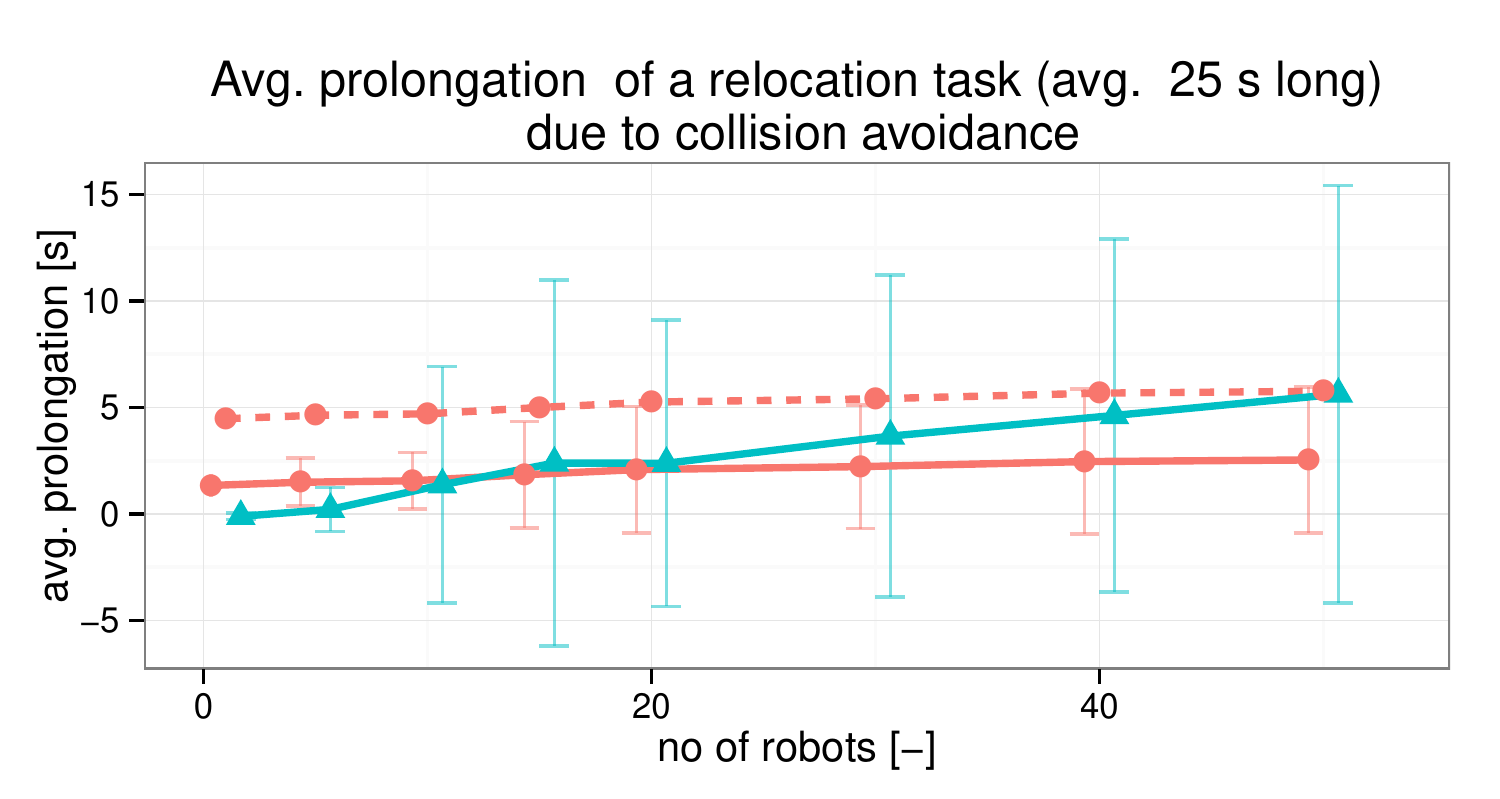} & \includegraphics[width=0.27\paperwidth]{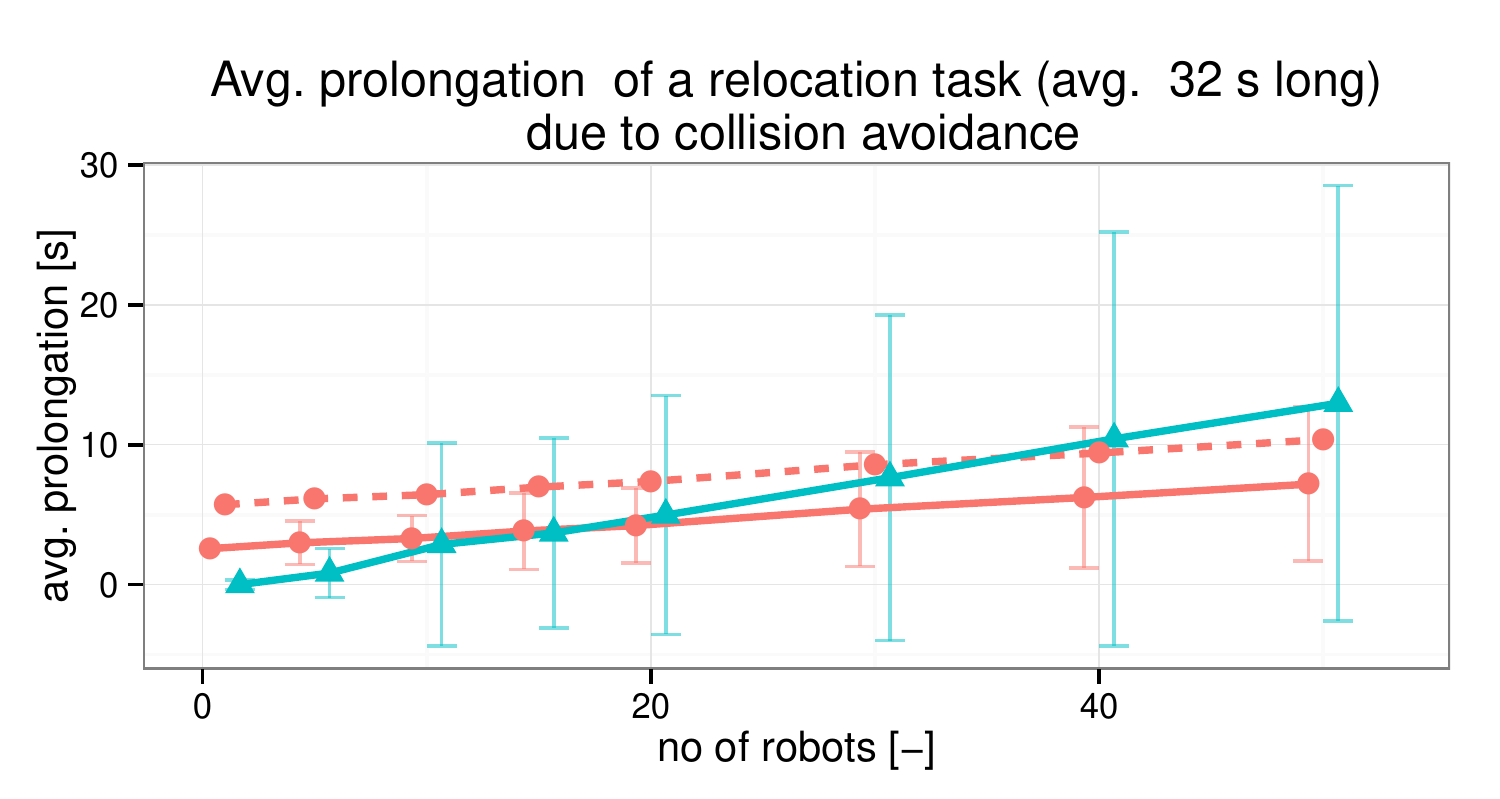} & \includegraphics[width=0.27\paperwidth]{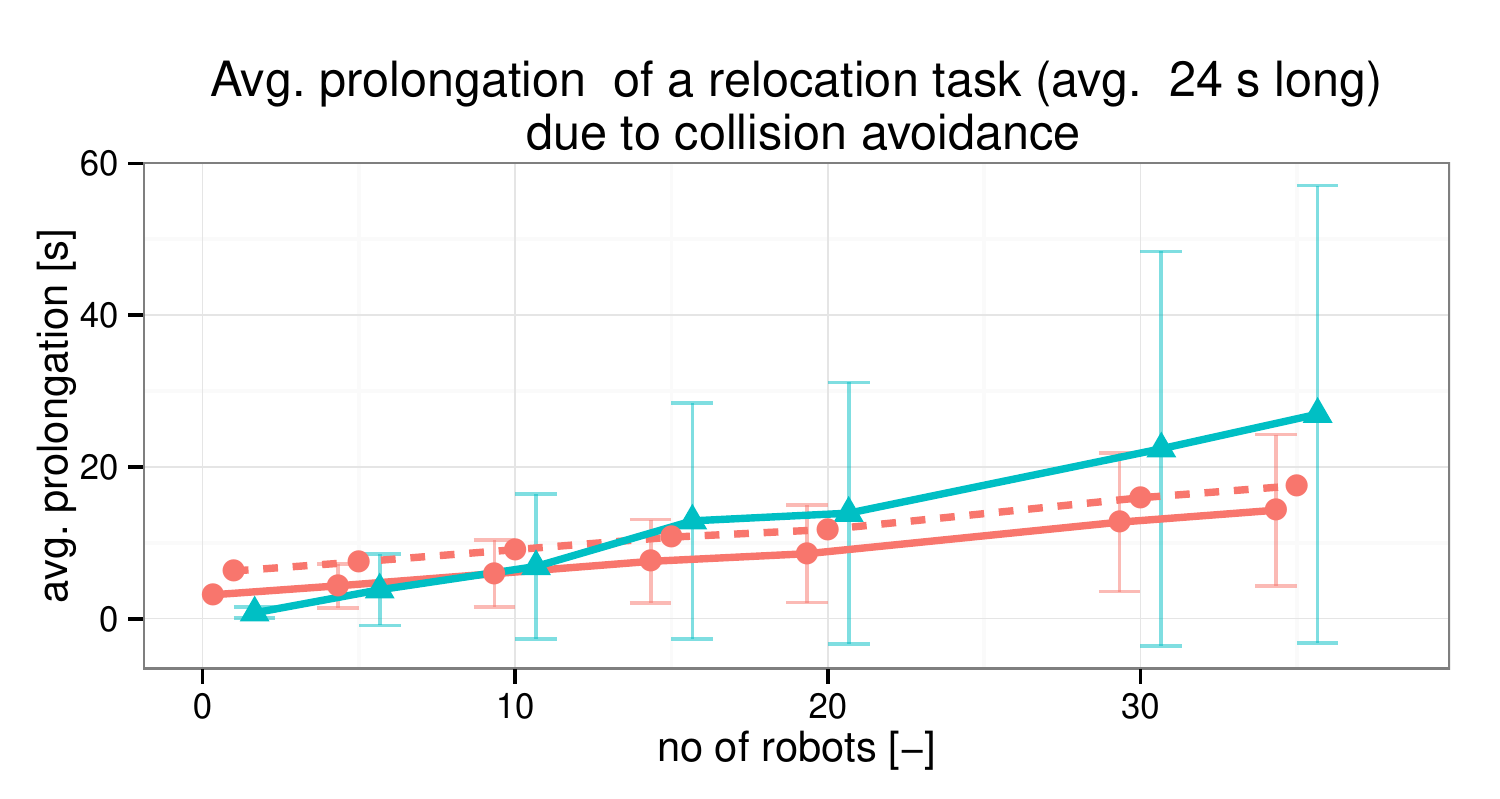}\tabularnewline
\multicolumn{3}{c}{\includegraphics[width=8cm]{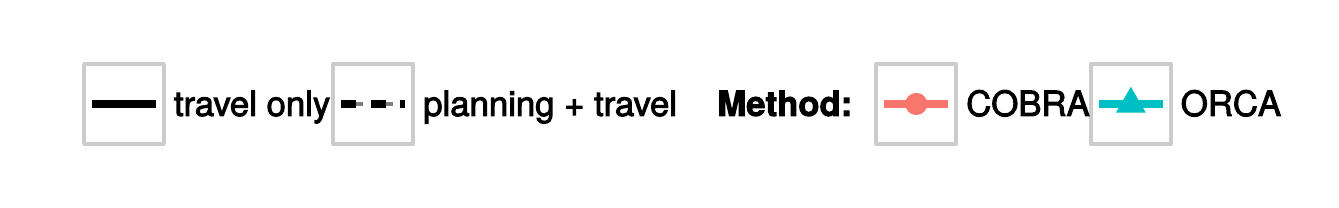}}\tabularnewline
\end{tabular}
\par\end{centering}

\caption{\label{fig:Results}Results. The bars represent standard deviation. }
\end{figure*}

In this section we compare the performance of COBRA and ORCA, a state-of-the
art decentralized approach for on-line collision avoidance in large
multi-robot teams. The two algorithms are compared in three real-world
environments shown in Figure~\ref{fig:Test-environments}. The test
environments are valid infrastructures. The execution of a multi-robot
system is simulated using a multi-robot simulator%
\footnote{The simulator and the test instances are available for download at
http://github.com/mcapino/cobra-icaps2015.%
}. During each simulation, the given number of robots is created and
each of them is successively assigned $4$ relocation tasks to a randomly
chosen unassigned endpoint. When the simulation is started, all the
robots are initialized and the first relocation task with random destination
endpoint is issued. To avoid the initial unnatural situation in which
all robots would need to plan simultaneously, the initial task is
issued with a random delay within the interval {[}0,~30\,s{]}. Once
the robot reaches the destination of its task, a new random destination
is generated and the process is repeated until the required number
of relocation tasks have been generated. For each such simulation,
we observe whether all robots successfully carried out all assigned
tasks and the time needed to reach the destination of each relocation
task. Further, we compute the prolongation introduced by collision
avoidance as $p=t^{A}-t'$, where $t^{A}$ is a duration of a particular
task when an algorithm $A$ is used for trajectory coordination, and
$t'$ is the time the robot needs to reach the destination without
collision avoidance simply by following the shortest path at the roadmap
at maximum speed. 

The robots are modeled as circular bodies with radius $r$=50\,cm
that can travel at maximum speed $v_{\mathrm{max}}$=1\,m/s. The
relative size of a robot and the environment in which the robots operate
is depicted in Figure~\ref{fig:Test-environments}, where the red
circles indicate the size of a robot. The roadmaps are constructed
as 8-connected grid with additional vertices and edges added near
the walls to maintain connectivity in narrow passages. The non-diagonal
edges of the base grid are 130\,cm long, the diagonal edges are 183\,cm
long.

The planning window used by COBRA is $t_{\mathrm{planning}}=\text{3\,\ s}$
long, yet on average single planning required only around 0.7\,s
even on the most challenging instances. The time-extended roadmap
uses discretization $\delta t$=650\, that conveniently splits travel
on the non--diagonal edges into 2 time steps and diagonal edges into
3 time steps.

The reactive technique ORCA~\cite{vanBerg2011_ORCA} is a control-engineering
approach typically used in a closed-loop such that at each time instant
it selects a collision-avoiding velocity vector from the continuous
space of robot's velocities that is the closest to the robot's desired
velocity. In our implementation, at each time instant the algorithm
computes a shortest path from the robots current position to its goal
on the same roadmap graph as is used for trajectory planning by COBRA.
The desired velocity vector then points at this shortest path at the
maximum speed. When using ORCA, we often witnessed dead-lock situations
during which the robots either moved at extremely slow velocities
or even stopped completely. If any of the robots did not reach its
destination in the runtime limit of 10\,mins (avg. task duration
is less than 33\,s), we considered the run as failed.

\subsubsection*{Results}

Figure~\ref{fig:Results} shows the performance of COBRA and ORCA
in the three test environments. The top row of plots shows the success
rate of each algorithm on instances with increasing number of robots.
In accordance with our theoretical results, we can see that the COBRA
algorithm reliably leads all robots to their assigned destinations
without collisions. When we tried to realize collision-free operation
using ORCA, the algorithm led in some cases the robots into a dead-lock.
The problem was exhibited more often in environments with narrow passages
as we can see in the success rate plot for the Office environment.

The bottom row of plots shows the average prolongation of a relocation
task when either COBRA and ORCA is used for collision avoidance. The
total prolongation introduced by COBRA is composed of two components:
planning time and travel time. When a new task is used, the robot
waits for $t_{\mathrm{planning}}=3\textrm{\,\ s}$ in order to plan
a collision-free trajectory to the destination of its new task. Only
then, the found trajectory is traveled by the robot until the desired
destination is reached. The robots controlled by ORCA start moving
immediately because they follow a precomputed policy towards the destination
of the current task or a collision-avoiding velocity if a possible
collision is detected. Recall that ORCA performs optimization in the
continuous space of robot's instantaneous velocities, whereas COBRA
plans a global trajectory on a roadmap graph with a discretized time
dimension. In the case of simple conflicts, ORCA can take advantage
of its ability to optimize in the continuous space and generates motions
where the robots closely pass each other, whereas COBRA has to stick
to the underlying discretization, which does not always allow such
close evasions. However, when the conflicts become more intricate,
the advantages of global planning starts to outweigh the potential
loss introduced by space-time discretization. The exact influence
of planning in a discretized space-time can be best observed by looking
at the data point for instances with 1 robot -- these instances do
not involve any collision avoidance and thus the prolongation can
be attributed purely to the discretization of space-time in which
the robot plans. Further, we can observe that the local collision
avoidance is less predictable, which is exhibited by the larger standard
deviation.

\section{Conclusion}

We proposed a novel method for on-line multi-robot trajectory planning
called Continuous Best-response Approach (COBRA) and both formally
and experimentally analyzed its properties. We have shown that the
algorithm has a unique set of features -- its time complexity is low
polynomial (quadratic in the number of robots) and yet it achieves
completeness in a class of environments called valid infrastructures
that encompass most human-made environments that have been intuitively
designed for efficient transport. Further, our technique is directly
applicable to systems with dynamically issued task and can be implemented
in a decentralized fashion on heterogeneous robots.  We experimentally
compared COBRA with a popular reactive technique ORCA in three real-world
maps using simulation. The results show that COBRA is more reliable
than ORCA and in more challenging scenarios, the planning approach
generates trajectories that are up to 48\,\% faster than ORCA.

\section{Acknowledgements}

This work was supported by the Grant Agency of the Czech Technical
University in Prague, grant No. SGS13/143/OHK3/2T/13 and by the Ministry
of Education, Youth and Sports of Czech Republic within the grant
no. LD12044.

\bibliographystyle{aaai}
\bibliography{bib}

\begin{thebibliography}{}

\bibitem[\protect\citeauthoryear{de Wilde, ter Mors, and
  Witteveen}{2013}]{deWilde_push_and_rotate_aamas}
de~Wilde, B.; ter Mors, A.~W.; and Witteveen, C.
\newblock 2013.
\newblock Push and rotate: cooperative multi-agent path planning.
\newblock In {\em Proceedings of the 2013 international conference on
  Autonomous agents and multi-agent systems},  87--94.
\newblock International Foundation for Autonomous Agents and Multiagent
  Systems.

\bibitem[\protect\citeauthoryear{Erdmann and
  Lozano-P{\'e}rez}{1987}]{Erdmann87onmultiple}
Erdmann, M., and Lozano-P{\'e}rez, T.
\newblock 1987.
\newblock On multiple moving objects.
\newblock {\em Algorithmica} 2:1419--1424.

\bibitem[\protect\citeauthoryear{Ghosh}{2010}]{ghosh2010DistributedSystems}
Ghosh, S.
\newblock 2010.
\newblock {\em Distributed systems: an algorithmic approach}.
\newblock CRC press.

\bibitem[\protect\citeauthoryear{Guy \bgroup et al\mbox.\egroup
  }{2009}]{Guy2009_ClearPath}
Guy, S.~J.; Chhugani, J.; Kim, C.; Satish, N.; Lin, M.; Manocha, D.; and Dubey,
  P.
\newblock 2009.
\newblock Clearpath: Highly parallel collision avoidance for multi-agent
  simulation.
\newblock In {\em Proceedings of the 2009 ACM SIGGRAPH/Eurographics Symposium
  on Computer Animation}, SCA '09,  177--187.
\newblock New York, NY, USA: ACM.

\bibitem[\protect\citeauthoryear{Hopcroft, Schwartz, and
  Sharir}{1984}]{hopcroft84}
Hopcroft, J.; Schwartz, J.; and Sharir, M.
\newblock 1984.
\newblock On the complexity of motion planning for multiple independent
  objects; pspace- hardness of the "warehouseman's problem".
\newblock {\em The International Journal of Robotics Research} 3(4):76--88.

\bibitem[\protect\citeauthoryear{Lalish and
  Morgansen}{2012}]{lalish2012distributed}
Lalish, E., and Morgansen, K.~A.
\newblock 2012.
\newblock Distributed reactive collision avoidance.
\newblock {\em Autonomous Robots} 32(3):207--226.

\bibitem[\protect\citeauthoryear{Spirakis and
  Yap}{1984}]{SpirakisY84_Strong_NP_Hardness_of_Moving_Many_Discs}
Spirakis, P.~G., and Yap, C.-K.
\newblock 1984.
\newblock Strong np-hardness of moving many discs.
\newblock {\em Inf. Process. Lett.} 19(1):55--59.

\bibitem[\protect\citeauthoryear{Standley}{2010}]{Standley10}
Standley, T.~S.
\newblock 2010.
\newblock Finding optimal solutions to cooperative pathfinding problems.
\newblock In Fox, M., and Poole, D., eds., {\em AAAI}.
\newblock AAAI Press.

\bibitem[\protect\citeauthoryear{Surynek}{2009}]{Surynek:2009:NAP:1703435.1703586}
Surynek, P.
\newblock 2009.
\newblock A novel approach to path planning for multiple robots in bi-connected
  graphs.
\newblock In {\em Proceedings of the 2009 IEEE international conference on
  Robotics and Automation}, ICRA'09,  928--934.
\newblock Piscataway, NJ, USA: IEEE Press.

\bibitem[\protect\citeauthoryear{Van Den~Berg \bgroup et al\mbox.\egroup
  }{2011}]{vanBerg2011_ORCA}
Van Den~Berg, J.; Guy, S.; Lin, M.; and Manocha, D.
\newblock 2011.
\newblock Reciprocal n-body collision avoidance.
\newblock {\em Robotics Research}  3--19.

\bibitem[\protect\citeauthoryear{Van~den Berg, Lin, and
  Manocha}{2008}]{vanDenBerg2008RVO}
Van~den Berg, J.; Lin, M.; and Manocha, D.
\newblock 2008.
\newblock Reciprocal velocity obstacles for real-time multi-agent navigation.
\newblock In {\em Robotics and Automation, 2008. ICRA 2008. IEEE International
  Conference on},  1928--1935.
\newblock IEEE.

\bibitem[\protect\citeauthoryear{\v{C}\'{a}p \bgroup et al\mbox.\egroup
  }{2013}]{cap_2013_b}
\v{C}\'{a}p, M.; Nov\'{a}k, P.; Seleck\'{y}, M.; Faigl, J.; and
  Vok\v{r}\'{i}nek, J.
\newblock 2013.
\newblock Asynchronous decentralized prioritized planning for coordination in
  multi-robot system.
\newblock In {\em Intelligent Robots and Systems (IROS), 2013}.

\bibitem[\protect\citeauthoryear{Velagapudi, Sycara, and
  Scerri}{2010}]{VelagapudiSS10}
Velagapudi, P.; Sycara, K.~P.; and Scerri, P.
\newblock 2010.
\newblock Decentralized prioritized planning in large multirobot teams.
\newblock In {\em IROS},  4603--4609.
\newblock IEEE.

\bibitem[\protect\citeauthoryear{Wagner and
  Choset}{2015}]{Wagner2014AIJ_subdimensionalExpansion}
Wagner, G., and Choset, H.
\newblock 2015.
\newblock Subdimensional expansion for multirobot path planning.
\newblock {\em Artificial Intelligence} 219:1 -- 24.

\end{thebibliography}

\end{document}